\documentclass[runningheads]{llncs}

\usepackage[T1]{fontenc}
\usepackage{amsmath,amssymb}
\usepackage{graphicx}
\usepackage{booktabs}
\usepackage{multirow}
\usepackage[hidelinks]{hyperref}

\graphicspath{{figures/}}
\renewcommand{\Re}{\operatorname{Re}}
\renewcommand{\Im}{\operatorname{Im}}
\newcommand{\ppl}{\mathrm{PPL}}

\begin{document}

\title{An Input-Dependent Readout Phase\\for Resonant Sequence Models}
\titlerunning{An Input-Dependent Readout Phase for Resonant Sequence Models}
\author{Shivesh Rane\orcidID{0009-0002-1265-1292}}
\authorrunning{S. Rane}
\institute{}

\maketitle

\begin{abstract}
Sequence models with a fixed-size state avoid the quadratic cost of attention,
but face a design choice. An input-dependent state transition gives
content-dependent memory at the cost of FFT-based training, while a fixed
transition keeps the $O(T\log T)$ convolutional path but reads memory the same
way for every input. ResonatorLM, a recent model of the second kind, reads a
bank of damped complex oscillators through a fixed learned phase per head. We
make that readout phase a function of the current token, which adds input
dependence to the read only, keeps the convolution exact, and costs 0.2\% more
parameters. An analysis of the untrained bank shows that phase does encode
elapsed time, but an ordinary token falls below the background within 2-8
steps and no choice of damping extends linear reconstruction beyond about 16
tokens. In 6M-parameter character-level models on WikiText-2, the
input-dependent phase matches the original at the training length and removes
its failure beyond it: trained on 256 tokens and tested on 1024, perplexity
falls from 12.81 to 3.76. We trace the failure to long-memory heads that
accumulate more history than training exposed them to. The learned phase tracks
the current character rather than distance and does not improve retrieval.
\end{abstract}

\section{Introduction}

Self-attention~\cite{vaswani2017attention} costs time quadratic in the context
length, which has made sequence models with a fixed-size recurrent state an
active alternative~\cite{gu2020hippo,gu2022s4,rusch2025linoss,gu2023mamba}.
These models share one decisive design choice: whether the state transition
depends on the input. Selective models such as Mamba and its
successors~\cite{gu2023mamba,mamba3,kimi2025linear,gateddeltanet2} make it
input-dependent and so decide per token what to store and what to forget. The
recurrence is then no longer a fixed convolution, and the $O(T\log T)$ FFT
training path is lost. Structured state-space models such as
S4~\cite{gu2022s4} keep the transition fixed. Each layer is a convolution with a
fixed kernel and trains with FFTs, but its memory is written and read
identically for every input.

We ask how much input dependence can be added to the fixed branch without
leaving it, by making only the \emph{read} of the state input-dependent. The
question has a concrete form in resonant sequence models.
ResonatorLM~\cite{chaudhury2026resonatorlm} replaces attention with banks of
damped complex oscillators. Every past token enters the state at an angle
proportional to its age, and the state is read by projecting onto one fixed,
learned direction per head. That direction selects which ages are read, and it
never changes after training. We let it depend on the current token.

\paragraph{Contributions.}
\begin{enumerate}
  \item An input-dependent readout phase for resonant models that keeps the
        kernel input-independent and the FFT path exact (Sect.~\ref{sec:method}),
        with a capacity identity showing that a fixed phase reads half of the
        stored state (Proposition~\ref{prop:capacity}).
  \item A measurement of what the published oscillator bank can hold
        (Sect.~\ref{sec:bank}): an ordinary token is undetectable after a few
        steps, and linear reconstruction stops at about 16 tokens whether the
        bank is damped, lengthened or undamped.
  \item Trained results (Sect.~\ref{sec:trained}): the input-dependent phase
        ties the original at the training length and removes its failure beyond
        it, a failure we trace causally to long-memory heads
        (Sect.~\ref{sec:mechanism}). The learned phase is a content-dependent
        read, not a pointer back in time (Sect.~\ref{sec:whatphase}).
\end{enumerate}

\section{Background: ResonatorLM}
\label{sec:background}

A ResonatorLM mixer has $H$ heads. Head $h$ has a damping $\alpha_h>0$, a
frequency $\omega_h\in(0,\pi)$ and a phase $\varphi_h$, all learned. The input
$x_t\in\mathbb{R}^{d}$ is projected to a drive $u_{t,h}\in\mathbb{R}^{d_h}$ and
a gate. Each head keeps a complex state $s_{t,h}\in\mathbb{C}^{d_h}$:
\begin{equation}
  s_{t,h} = \lambda_h\, s_{t-1,h} + u_{t,h},
  \qquad \lambda_h = e^{-\alpha_h + i\omega_h},
  \qquad\text{so}\qquad
  s_{t,h} = \sum_{j\le t} e^{-\alpha_h (t-j)}\, e^{i\omega_h (t-j)}\, u_{j,h}.
  \label{eq:recurrence}
\end{equation}
A token of age $\tau=t-j$ is stored with magnitude $e^{-\alpha_h\tau}$ and angle
$\omega_h\tau$. The head outputs the real projection
\begin{equation}
  y_{t,h} = \Re\!\left(e^{-i\varphi_h} s_{t,h}\right)
          = \sum_{j\le t} k_h[t-j]\,u_{j,h},
  \qquad k_h[\tau] = e^{-\alpha_h\tau}\cos(\omega_h\tau - \varphi_h),
  \label{eq:readout}
\end{equation}
so a layer is a causal convolution with a fixed kernel, computed by FFT in
$O(T\log T)$ for training and by Eq.~\eqref{eq:recurrence} for decoding. The
output is gated, mixed across heads, and added to a short causal convolution.
Since $\varphi_h$ is a constant, the readout keeps the component of the state
along one fixed direction and discards the orthogonal one. The gate acts after
this projection and cannot recover what it discards.\footnote{The original paper
writes the kernel with $\cos(\omega t+\varphi)$ but a readout that unrolls to
$\cos(\omega t-\varphi)$. We use the latter throughout.}

\section{Method}
\label{sec:method}

\begin{definition}[Input-dependent readout phase]
For head $h$ and position $t$, let
\begin{equation}
  \varphi_{t,h} = \varphi_h + \pi\tanh\!\left(w_h^{\top}\bar x_t + b_h\right),
  \qquad
  y_{t,h} = \cos\varphi_{t,h}\,\Re s_{t,h} + \sin\varphi_{t,h}\,\Im s_{t,h},
  \label{eq:dynphase}
\end{equation}
where $\bar x_t$ is the normalised mixer input and $(w_h,b_h)$ are learned.
\end{definition}

The transition $\lambda_h$ is unchanged, so both state components are still
causal convolutions with fixed kernels,
\begin{equation}
  \Re s_{t,h} = \sum_{\tau\ge0} e^{-\alpha_h\tau}\cos(\omega_h\tau)\,u_{t-\tau,h},
  \qquad
  \Im s_{t,h} = \sum_{\tau\ge0} e^{-\alpha_h\tau}\sin(\omega_h\tau)\,u_{t-\tau,h},
  \label{eq:components}
\end{equation}
which share one forward FFT of $u$. Equation~\eqref{eq:dynphase} is then applied
pointwise. The FFT and recurrent paths agree to $4.4\times10^{-16}$ in float64.
The projection adds $H(d+1)$ parameters per layer, 12{,}336 in total
($+0.2\%$), and is zero-initialised, so training starts exactly at the
fixed-phase model.

The following identity quantifies what a fixed phase gives up.

\begin{proposition}[Capacity]
\label{prop:capacity}
Let the drive of one channel be white noise with unit variance, and let
$z_t\in\mathbb{R}^m$ be linear features of the state with nonsingular covariance
$\Sigma$. Write $R^2(d)$ for the fraction of the variance of $u_{t-d}$ that is
linearly recoverable from $z_t$. Then $\sum_{d\ge0} R^2(d) = m$. For the full
state of $H$ heads ($z=[\Re s,\Im s]$) this is $2H$; for any fixed real readout
it is $H$.
\end{proposition}
\begin{proof}
With $g_d=\mathrm{Cov}(z_t,u_{t-d})$, $R^2(d)=g_d^{\top}\Sigma^{-1}g_d$ and
$\Sigma=\sum_d g_d g_d^{\top}$, so
$\sum_d R^2(d)=\mathrm{tr}\big(\Sigma^{-1}\sum_d g_dg_d^{\top}\big)=\mathrm{tr}(I_m)=m$.
\qed
\end{proof}

A fixed phase therefore leaves half of the recoverable information unread
(numerically $8.00$ of $16.00$ for the published bank). Choosing the phase per
lag recovers essentially the full-state value, but the total is spread thinly
over lags: $R^2$ falls below 0.5 by $d=5$ and below 0.1 beyond $d=16$.

\section{What the Oscillator Bank Can Hold}
\label{sec:bank}

We first analyse the published bank without training: $H=8$ heads, half-lives
$\ln 2/\alpha_h$ log-spaced over $[2,2048]$, frequencies log-spaced over
$[\pi/1024,\,0.99\pi]$, $d_h=32$, float64.

\paragraph{Detectability.} The bank is driven by white noise for $T=4160$ steps,
and one marked token of typical norm ($\sqrt{d_h}$) is added $d$ steps back.
Reading at the angle that token occupies, $\varphi=\omega_h d$, we measure
\begin{equation}
  \mathrm{SNR}_h(d) = \frac{\mathbb{E}\,\lvert r_h(s)\rvert}
                           {\mathrm{sd}\,\big(r_h(s_{\mathrm{bg}})\big)},
  \qquad r_h(s)=\big\langle \Re(e^{-i\omega_h d}s_h),\,\hat u\big\rangle ,
  \label{eq:snr}
\end{equation}
where $s_{\mathrm{bg}}$ is the state without the token and $\hat u$ its
direction. With no token, $\mathrm{SNR}=\sqrt{2/\pi}\approx0.80$ exactly.
Figure~\ref{fig:snr} shows that the token exceeds $\mathrm{SNR}=3$ only for
$d\le2$, and heads with half-life $\ge105$ remain near the null even for $d=1$.
Their background grows from 1.4 to 37.3 in standard deviation as the half-life
lengthens: the damping that buys range also accumulates more background.

\begin{figure}[t]
  \centering
  \includegraphics[width=\textwidth]{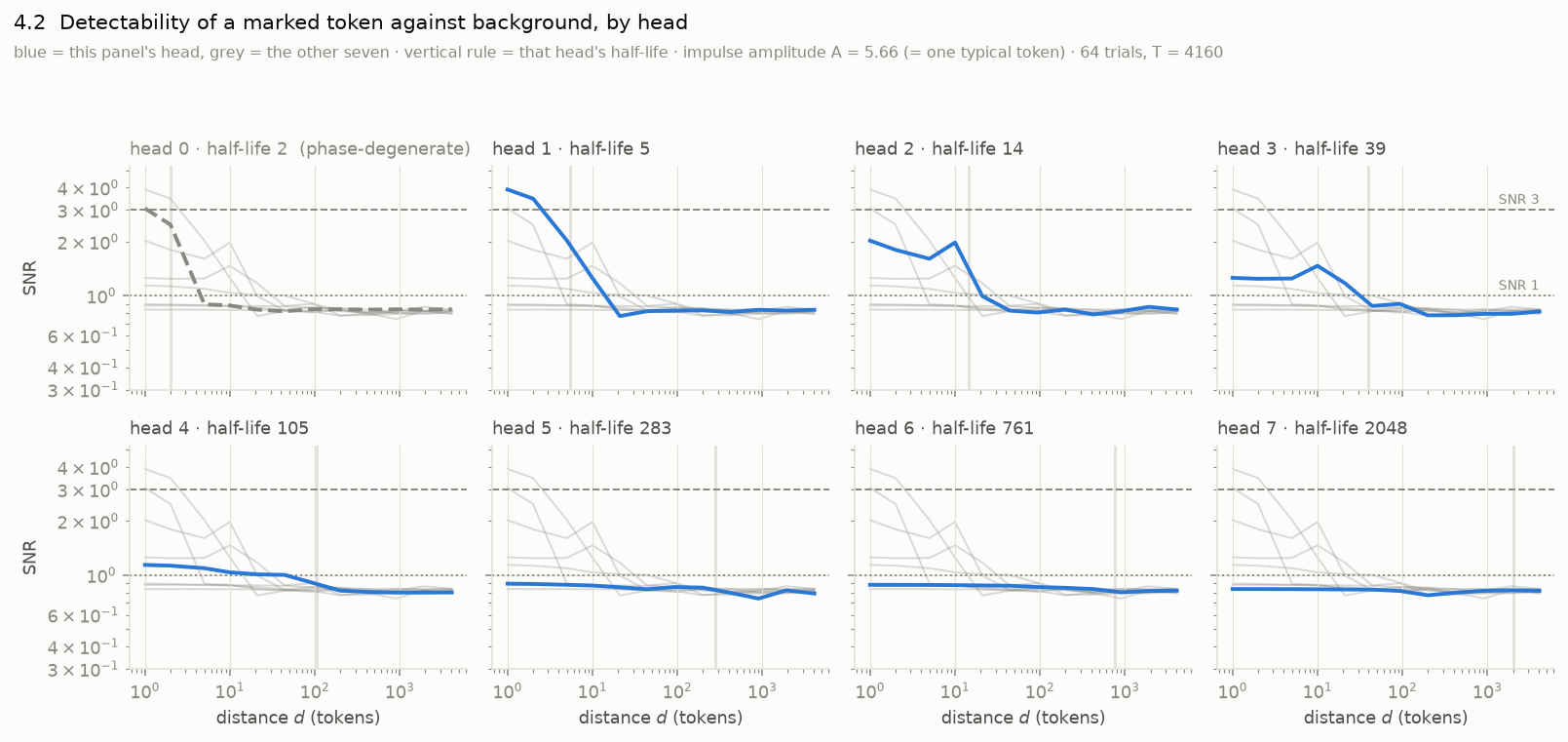}
  \caption{Detectability of one typical token against per-step background,
  per head (Eq.~\ref{eq:snr}). Dashed lines mark $\mathrm{SNR}=1$ and $3$.}
  \label{fig:snr}
\end{figure}

\paragraph{Reconstruction from text.} We drive the bank with WikiText-2
characters~\cite{merity2017wikitext} through a fixed random embedding, and train
a linear decoder on $[\Re s_t,\Im s_t]$ to predict the character at $t-d$. States
are read at $t\in[512,1024)$, data are split by sequence, and the decoder is
early-stopped. Five banks are compared: the published one, half-lives up to
$6{,}931$ (the model's floor on $\alpha$), four heads nearly undamped, all heads
undamped, and the six banks a trained model learned. All damped banks give the
same curve (Fig.~\ref{fig:recon}): above 98\% to $d=4$, 76\% at $d=8$, about
50\% at $d=12$, and chance by $d=32$--$64$. The undamped bank is at chance even at
$d=0$, because no token dominates a sum in which every past token has equal
weight. The limit is crowding, not damping.

\begin{figure}[t]
  \centering
  \includegraphics[width=0.8\textwidth]{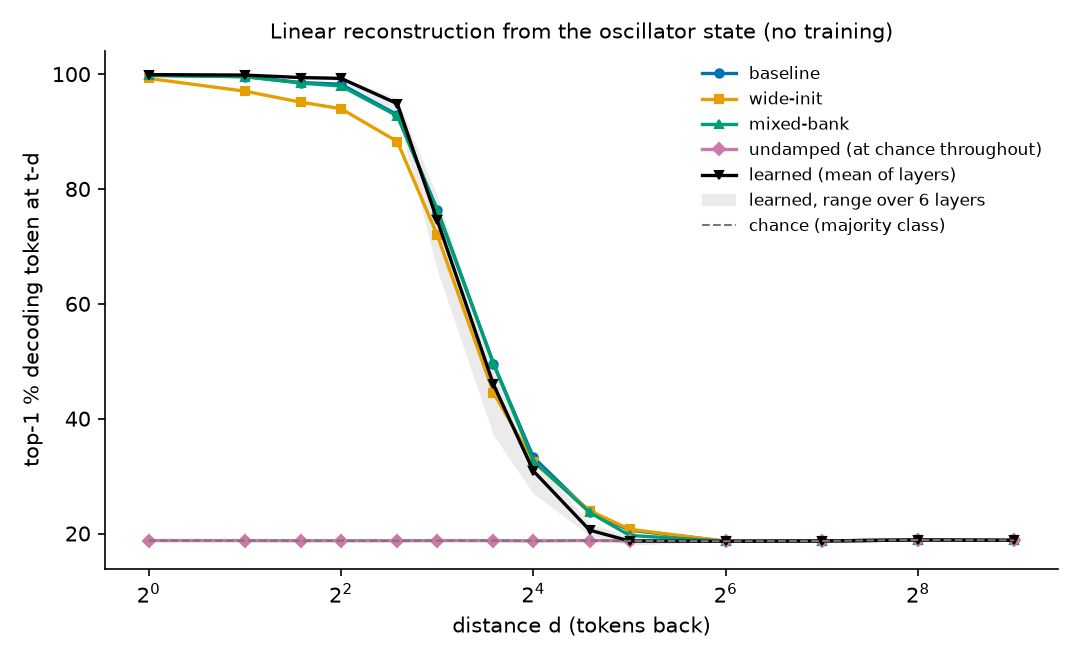}
  \caption{Linear reconstruction of the character $d$ steps back from the
  oscillator state. Damped banks coincide; the undamped bank is unreadable.}
  \label{fig:recon}
\end{figure}

Two further checks support this picture. With no background, the eight heads
jointly locate a token's age exactly to $d\ge32{,}768$, sixteen times the longest
period, so the time information is present. With background, however, an
ordinary token is located only to $d\approx8$. Repeating both measurements with
the bfloat16 roundings used in training changes the SNR by at most $0.09\%$ and
reconstruction by at most 0.65 points. An input-dependent phase can therefore
choose \emph{what} to read, but not recover what the state does not hold.

\section{Trained Models}
\label{sec:trained}

\paragraph{Setup.} We follow the ResonatorLM protocol: WikiText-2 raw at
character level (vocabulary 1013), 6 layers, $d=256$, $H=8$, $d_h=32$, 10{,}000
AdamW steps (learning rate $5\times10^{-4}$, cosine schedule, weight decay 0.1)
with 8{,}192 tokens per step, and bfloat16. The models compared are:
\textbf{A}, the original (fixed $\varphi_h$, initialised at 0; 6 seeds);
\textbf{C}, the proposal (Eq.~\ref{eq:dynphase}; 3 seeds); two fixed-phase
controls (3 seeds each), \textbf{B} with $\varphi_h$ initialised uniformly on
$(-\pi,\pi)$ and \textbf{E} with only the two longest-memory heads initialised
at $\pi/2$; and matched transformers ($d=248$) with rotary
(RoPE)~\cite{su2021roformer} or learned position embeddings (6 seeds each). A and
C are trained at $T\in\{256,512,1024\}$ and tested on the test split at $T$,
$2T$ and $4T$. A run \emph{collapses} if $\ppl(4T)>1.5\,\ppl(T)$.

\paragraph{At the training length.} Table~\ref{tab:quality} shows that C and A
tie at every training length. Both improve on the published ResonatorLM
numbers, and at 256 the RoPE transformer is better than either, contrary to
the original report.

\begin{table}[t]
  \centering
  \caption{Test perplexity and accuracy at the training length (mean $\pm$ sd
  over seeds). Published values from~\cite{chaudhury2026resonatorlm}.}
  \label{tab:quality}
  \setlength{\tabcolsep}{5pt}
  \begin{tabular}{llccc}
    \toprule
    Trained and tested at & Model & Seeds & Perplexity & Accuracy (\%) \\
    \midrule
    \multirow{5}{*}{256}
      & Transformer (RoPE)           & 6 & $3.426\pm0.007$ & $65.40$ \\
      & Transformer (learned PE)     & 6 & $3.444\pm0.016$ & $64.62$ \\
      & A: ResonatorLM               & 6 & $3.506\pm0.011$ & $64.76$ \\
      & C: input-dependent phase     & 3 & $3.495\pm0.004$ & $65.01$ \\
      & ResonatorLM, published       & 6 & $3.764$         & $61.31$ \\
    \midrule
    \multirow{3}{*}{512}
      & A                            & 3 & $3.421$ & $65.40$ \\
      & C                            & 3 & $3.414$ & $65.53$ \\
      & ResonatorLM, published       & 3 & $4.339$ & $57.20$ \\
    \midrule
    \multirow{3}{*}{1024}
      & A                            & 3 & $3.384$ & $65.74$ \\
      & C                            & 3 & $3.390$ & $65.76$ \\
      & ResonatorLM, published       & 3 & $4.502$ & $56.17$ \\
    \bottomrule
  \end{tabular}
\end{table}

\paragraph{Beyond the training length.} Table~\ref{tab:ratio} and
Fig.~\ref{fig:length} give the main result. C is lower than A in all six
settings beyond the training length, with non-overlapping seed ranges in each.
The effect is largest when training is short. Trained at 256 and tested at
1024, A collapses in all 6 seeds (perplexity 10.3--14.7, accuracy 45.0\%) while
C stays at 3.70--3.84 (accuracy 63.5\%). Trained at 512, A degrades by 40\% at
$4T$ and C by 3\%. Trained at 1024, both stay within 3\%. The transformers
degrade as well, most severely with learned positions.

\begin{table}[t]
  \centering
  \caption{Test perplexity when testing beyond the training length $T$ (mean
  over seeds). The tested length is shown in each block. Ratio is
  $\ppl(4T)/\ppl(T)$; ``collapsed'' counts seeds with ratio $>1.5$.}
  \label{tab:ratio}
  \setlength{\tabcolsep}{5pt}
  \begin{tabular}{llccccc}
    \toprule
    Trained at & Model & \multicolumn{3}{c}{Tested at} & Ratio & Collapsed \\
    \midrule
    \multirow{7}{*}{$T=256$}
      & & \emph{256} & \emph{512} & \emph{1024} & & \\
      & Transformer (RoPE)       & 3.43 & 4.62  & 10.39 & 3.03  & -- \\
      & Transformer (learned PE) & 3.44 & 11.33 & 50.15 & 14.56 & -- \\
      & A: ResonatorLM           & 3.51 & 3.57  & 12.81 & 3.65  & 6/6 \\
      & B: fixed, spread init    & 3.49 & 3.52  & 8.85  & 2.54  & 1/3 \\
      & E: long heads at $\pi/2$ & 3.49 & 3.47  & 7.57  & 2.17  & 3/3 \\
      & C: input-dependent phase & 3.50 & 3.46  & \textbf{3.76} & \textbf{1.07} & 0/3 \\
    \midrule
    \multirow{3}{*}{$T=512$}
      & & \emph{512} & \emph{1024} & \emph{2048} & & \\
      & A & 3.42 & 3.56 & 4.78 & 1.40 & 0/3 \\
      & C & 3.41 & 3.41 & \textbf{3.52} & \textbf{1.03} & 0/3 \\
    \midrule
    \multirow{3}{*}{$T=1024$}
      & & \emph{1024} & \emph{2048} & \emph{4096} & & \\
      & A & 3.38 & 3.43 & 3.50 & 1.03 & 0/3 \\
      & C & 3.39 & 3.41 & \textbf{3.47} & \textbf{1.02} & 0/3 \\
    \bottomrule
  \end{tabular}
\end{table}

\begin{figure}[t]
  \centering
  \includegraphics[width=\textwidth]{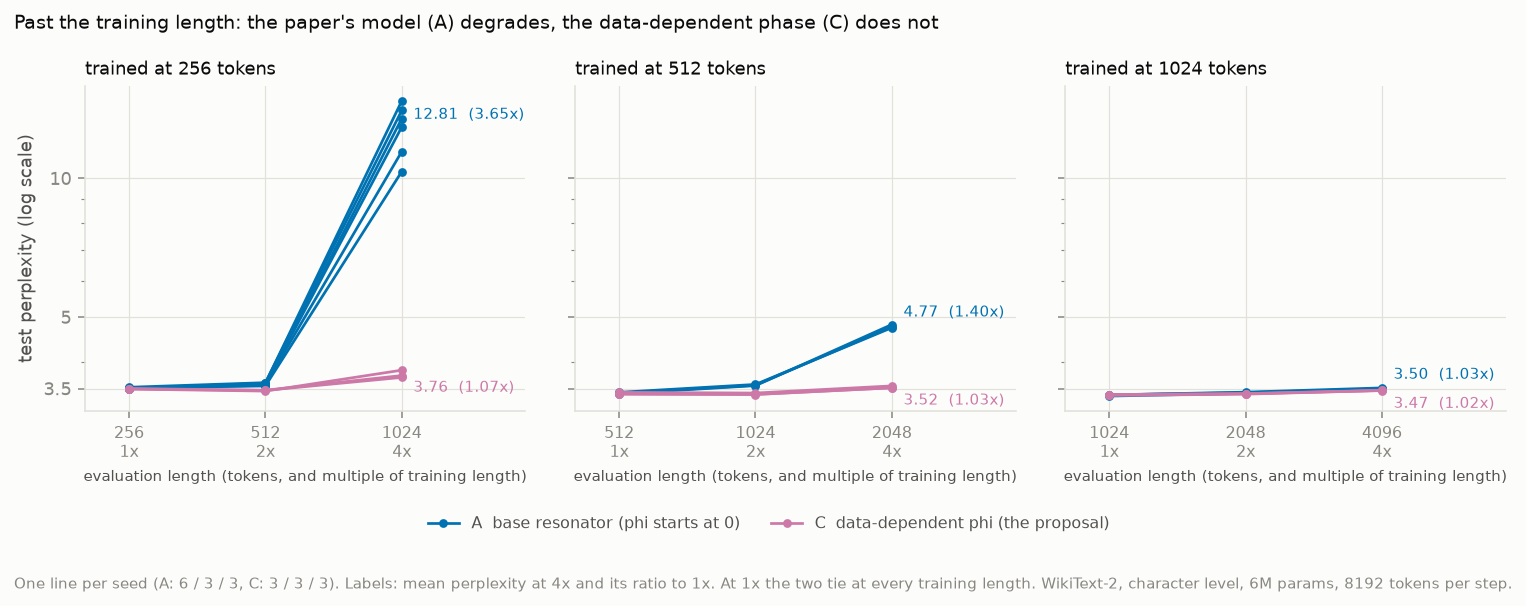}
  \caption{Perplexity at $1\times$, $2\times$ and $4\times$ the training length,
  one line per seed. A degrades beyond its training length, most when trained
  short; C does not.}
  \label{fig:length}
\end{figure}

\paragraph{Controls.} A fixed phase with a different initialisation does not
reliably reproduce the effect (Table~\ref{tab:ratio}, $T=256$). B collapses in
one of three seeds (18.79) and avoids it in two (4.22, 3.54). E keeps its two
longest heads near $\pi/2$ throughout training ($\lvert\cos\varphi\rvert\approx0.2$)
yet collapses in all three. Across all runs trained at 256, 10 of 12 fixed-phase
runs collapse at 1024, against 0 of 3 for C. Three seeds cannot show that input
dependence is necessary, but it was the only variant that never collapsed.

\subsection{Why the Original Model Collapses}
\label{sec:mechanism}

We evaluated the 15 models trained at 256 on the same 1024-token sequences,
measuring internal signals by position, and then modified the models at test
time only.

\paragraph{Where and what.} For A, perplexity is 3.40--3.45 on positions 0--255
and 3.52--3.64 on 256--511, then 18.4--41.7 on 512--767 and 48.8--76.0 on
768--1023. Among about 150 per-layer and per-head measurements, those that best
separate collapsing from healthy runs all concern long-memory heads (learned
half-life above about 200) in the upper layers. For example, the RMS state of
head 6 in layer 4 grows by at least $3.7\times$ from the first to the last
quarter of the sequence in every collapsing run, and by at most $2.0\times$ in
every healthy run (means: A $5.5\times$, C $1.7\times$). The upper layers thus
receive long-memory signals several times larger than any seen in training.

\paragraph{Causal test.} Table~\ref{tab:mechanism} confirms the cause.
Truncating only the kernels of heads with half-life above 256 to 256 lags
removes the collapse (A: $12.70\to3.43$; E: $7.57\to3.41$), and cannot change
anything at 256. Removing the constant component of the drive, whose
accumulation in a slow head rotates with absolute position, repairs most of A's
collapse but none of E's, so it is not the general cause.

\begin{table}[t]
  \centering
  \caption{Test-time interventions on models trained at 256: perplexity at 256 /
  at 1024 (mean over seeds 0--2). Truncation cuts a head's kernel to 256 lags.}
  \label{tab:mechanism}
  \setlength{\tabcolsep}{6pt}
  \begin{tabular}{lccc}
    \toprule
    Intervention & A & E & C \\
    \midrule
    None                                   & 3.48 / 12.70 & 3.49 / 7.57 & 3.49 / 3.76 \\
    Truncate heads with half-life $>256$   & 3.48 / \textbf{3.43} & 3.49 / \textbf{3.41} & 3.49 / 3.44 \\
    Truncate all heads                     & 3.48 / 3.40 & 3.49 / 3.41 & 3.49 / 3.41 \\
    Remove constant drive component        & 3.65 / 4.69 & 3.58 / 8.00 & 3.61 / 3.75 \\
    \bottomrule
  \end{tabular}
\end{table}

\paragraph{Consequence.} The collapse has two remedies. Truncation is free and
on WikiText-2 slightly better than C at 1024 (3.43 vs.\ 3.76), but it discards
all history older than the training length. That costs nothing here, since
Sect.~\ref{sec:bank} shows that no linearly recoverable information survives
that long. The input-dependent phase avoids the collapse with the full memory
retained, because its long-memory states drift far less. Whether the retained
memory helps requires a task that depends on older information, which we leave
to future work.

\subsection{What the Learned Phase Does}
\label{sec:whatphase}

Figure~\ref{fig:phi} analyses the learned offset
$\delta_{t,h}=\varphi_{t,h}-\varphi_h$ of C on test data.

\begin{itemize}
  \item \emph{It varies.} Its standard deviation is 0.78--0.98 rad in every
        layer, and $\tanh$ saturates on fewer than 3\% of tokens.
  \item \emph{It tracks content, not distance.} In layers 1--5 the current
        character explains 31--58\% of its variance, position in the word
        17--33\%, and distance to the previous occurrence of the same character
        4--8\%. The implied lookback $(\varphi_{t,h}\bmod2\pi)/\omega_h$ has
        rank correlation between $-0.05$ and $0.07$ with that distance.
  \item \emph{It is used.} Fixing it at $\varphi_h$ raises perplexity from 3.49
        to 5.74--8.54, and shuffling it across tokens to 7.16--11.15.
  \item \emph{It does not retrieve more.} A linear probe decoding the character
        $d$ steps back from the final hidden state scores 77.1/66.7/45.0/18.6\%
        for C against 78.7/70.0/50.4/20.2\% for A at $d=1,2,4,8$.
\end{itemize}

\begin{figure}[t]
  \centering
  \includegraphics[width=\textwidth]{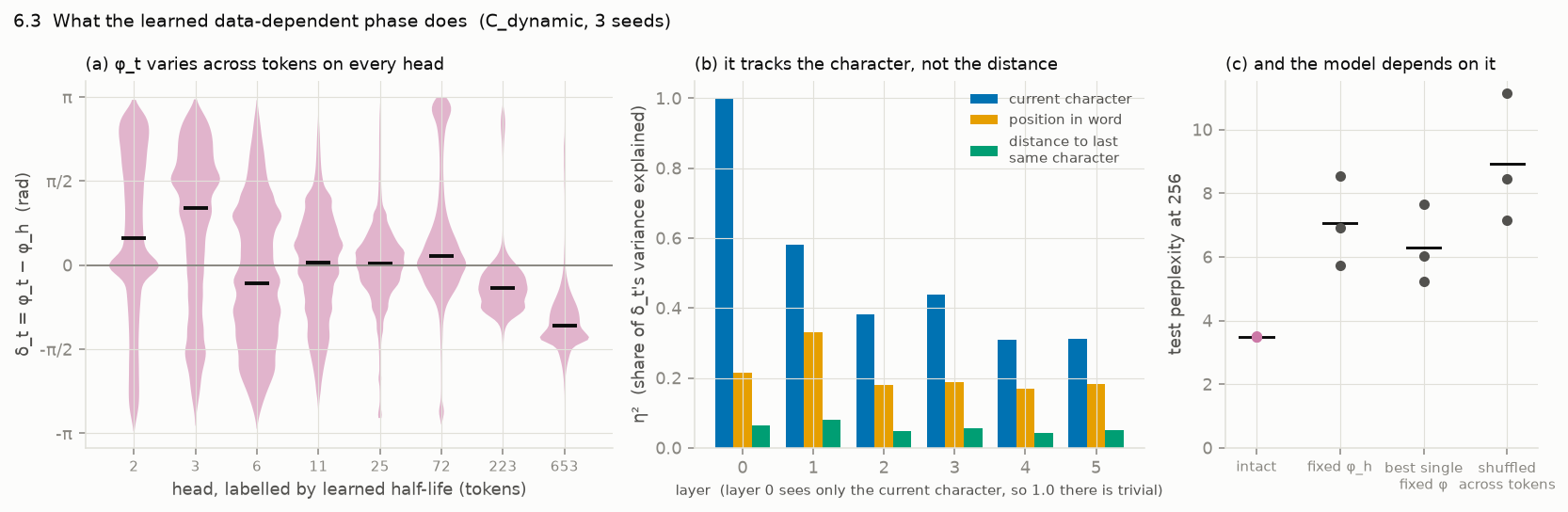}
  \caption{The learned phase of C (3 seeds). (a) Offset distribution per head.
  (b) Variance explained by the current character, position in word, and
  distance to the last same character. (c) Perplexity with the phase intact,
  fixed, set to the best constant, or shuffled.}
  \label{fig:phi}
\end{figure}

The phase therefore does not act as a pointer back in time. By
Eq.~\eqref{eq:dynphase} it chooses, per token, which mix of $\Re s$ and $\Im s$
to read, and so reaches over a sequence the half of the state that
Proposition~\ref{prop:capacity} shows a fixed phase leaves unread. Causality is
preserved (maximum prefix change under a suffix perturbation
$2.6\times10^{-14}$ in float64), and measured throughput is $0.98\times$ that of
A.

\section{Related Work}

\paragraph{Oscillatory and structured state spaces.} LinOSS~\cite{rusch2025linoss}
studies oscillatory state-space models and shows that undamped oscillators can
be stable and expressive. Section~\ref{sec:bank} shows that, under a linear
readout of a small state, removing damping destroys rather than extends what can
be read. HiPPO and S4~\cite{gu2020hippo,gu2022s4} give principled alternatives
to log-spaced initialisation.

\paragraph{Input-dependent dynamics.} Mamba~\cite{gu2023mamba},
Mamba-3~\cite{mamba3}, Kimi Linear~\cite{kimi2025linear} and Gated
DeltaNet-2~\cite{gateddeltanet2} make the transition input-dependent and give up
the fixed convolution. We place input dependence only in the read, which keeps
the convolution but cannot create memory the write side does not keep.

\paragraph{Recall and length extrapolation.} Recall in efficient models is
bounded by state size~\cite{arora2023zoology,arora2024based}. Length
extrapolation in transformers depends on position
encoding~\cite{press2022alibi,su2021roformer}; our RoPE and learned-position
baselines show the same sensitivity.

\section{Limitations and Conclusion}

All results use 6M-parameter character-level models on WikiText-2 with three
seeds for the proposal. We did not test the original paper's transfer datasets
or speed benchmarks, and on an associative-recall task only one of three seeds
of the proposal learned the task, so we draw no retrieval conclusion. A free
test-time truncation also fixes the original's collapse on this benchmark; the
proposal's advantage over it, retaining older history, is untested.

Making the readout phase of ResonatorLM input-dependent keeps the convolutional
training path, adds 0.2\% parameters, matches the original at every training
length, and removes its failure beyond the training length, which we trace to
long-memory heads accumulating more history than training exposed them to. The
benefit does not come from reaching further back: an ordinary token is lost in
the oscillator state within a few steps whatever the damping, and the learned
phase instead chooses, per character, which part of the state to read.

\bibliographystyle{splncs04}
\bibliography{refs}

\end{document}